\documentclass{article}

\usepackage[preprint]{neurips_2019}

\usepackage[utf8]{inputenc} %
\usepackage[T1]{fontenc}    %
\usepackage{hyperref}       %
\usepackage{url}            %
\usepackage{booktabs}       %
\usepackage{amsfonts}       %
\usepackage{nicefrac}       %
\usepackage{microtype}      %
\usepackage{color}
\usepackage{amsmath,amsthm,amssymb}
\usepackage{cleveref}
\usepackage{graphicx}
\usepackage{lscape}
\usepackage{todonotes}
\usepackage{subcaption}

\newcommand{\alg}{{FG}}
\newcommand{\eg}{{\emph{e.g., }}}
\newcommand{\ie}{{\emph{i.e., }}}
\newcommand{\algname}{{Feature Gradients: Scalable Feature Selection via Discrete Relaxation}}

\newtheorem{proposition}{Proposition}[section]
\newtheorem{lemma}{Lemma}[section]

\DeclareMathOperator*{\argmin}{arg\,min}
\crefformat{equation}{(#2#1#3)}
\graphicspath{{figs/}}

\title{\large\algname}

\author{%
    Rishit Sheth, Nicolo Fusi \\
    Microsoft Research New England \\
    Cambridge, MA 02142 \\
    \texttt{\{rishet,fusi\}@microsoft.com} \\
}

\begin{document}

\maketitle

\begin{abstract}
    In this paper we introduce Feature Gradients, a gradient-based search algorithm for feature selection.
    Our approach extends a recent result on the estimation of learnability in the sublinear data regime by showing that the calculation can be performed
    iteratively (\ie in mini-batches) and in linear time and space with respect to both the number of features $D$ \emph{and} the sample size $N$.
    This, along with a discrete-to-continuous relaxation of the search domain, allows for an efficient, gradient-based search algorithm among feature subsets for very large datasets.
    Crucially, our algorithm is capable of finding higher-order correlations between features and targets for both the $N>D$ and $N<D$ regimes, as opposed to approaches that do not consider such interactions and/or only consider one regime.
    We provide experimental demonstration of the algorithm in small and large sample- and feature-size settings.

\end{abstract}

\section{Introduction}
We consider the problem of feature selection in supervised learning tasks.
Feature selection remains a crucial step in machine learning pipelines and continues to see active research (e.g., \cite{Aghazadeh2018mission,Abid2019concrete}).
Earlier application of feature selection methods, such as those in computational biology, considered settings where the number of features available was large and the sample size was relatively small. The goal was to choose a small subset of features with good explanatory power and doing so required \emph{statistically efficient} feature selection methods. 
More recent applications of feature selection methods (\eg in natural language processing) are characterized both by large sample sizes and large feature spaces, making \emph{computational efficiency} another key requirement.

To address these challenges, several feature selection methods have been proposed over the years. These methods can be broadly categorized as either wrapper or filter methods.
Wrapper methods utilize a prescribed model to select features by training and evaluating the model on different feature subsets \citep{Kohavi1997wrappers}.
Filter methods utilize a computationally cheap evaluation criterion to rank features individually or in certain combinations \citep{Yu2003feature}.
In addition to hybrid combinations of these basic approaches, embedded feature selection methods that perform joint model selection and training have also been proposed. An example of such hybrid methods is the lasso \citep{Tibshirani1996regression}.
Generally, wrapper methods give better results since they directly evaluate the final supervised learning task.
However, they are also the most computationally demanding, since they require training and evaluating a model per proposed feature subset.
Filter methods are considered computationally cheaper alternatives with better scaling, but lack a principled means of considering feature interactions without either dramatically increasing the computational cost (e.g., by directly computing pairwise and higher-order interaction terms) or employing a heuristic search among feature subsets.

In this paper, our goal is to develop a feature selection method that is statistically efficient (\ie identifies a subset of features highly correlated with the target in $N < D$ settings), and computationally efficient (\ie scales to millions of samples and features).
Ideally, the method should consist of primitive operations that can be accelerated using GPUs.

We build our method starting from a recent result by \cite{Kong2018estimating} in which the authors present an estimator for the accuracy induced by a set of features in a linear model.
While this estimator has desirable statistical properties, it can only assess the ``quality'' of a given set of features and doesn't prescribe a procedure for selecting a specific set of features.
In the following, we treat feature selection as a combinatorial optimization problem and employ a discrete relaxation
to efficiently traverse the space of possible feature sets.
Our algorithm combines the strong statistical properties of the \cite{Kong2018estimating} estimator with the ability to search over complex optimization spaces via standard gradient descent.
Specifically, the method we propose has the following properties:
\begin{itemize}
    \item \emph{Low computational complexity}. The complexity scales linearly both with the sample size and the number of features, enabling feature selection even in large datasets.
    \item \emph{High statistical efficiency}. The estimator accuracy degrades sublinearly with increasing dimension allowing for accurate search over large featurization spaces consisting of millions of features. %
        For example, increasing the number of features by a factor $\gamma>1$ only requires a factor of $\sqrt{\gamma}$ increase in the number of samples to maintain the same level of estimator accuracy.
    \item \emph{Efficient detection of higher-order interactions}. 
        Modeling higher-order feature interactions incurs a constant increase in time and space complexity with with each additional order.
\end{itemize}

\section{Feature Gradients}
\label{sec:fg}

Our algorithm is based on recent work on the learnability of linear models by \cite{Kong2018estimating} which provides high probability bounds on the performance of linear predictors in high-dimensional regression (residual variance) and classification (accuracy).
In particular, the bounds for the estimators they introduce have a sublinear dependence on dimension, highlighting the possibility of making well-informed decisions on which feature subset would be most useful for a dataset on a task with relatively little data.
We start by giving a brief description of these estimators, moving on to show how they can be modified and utilized for feature selection by means of optimization over a continuous space.
For clarity of exposition, we focus on the regression case, but the same ideas and approach hold for classification.

\newcommand{\xd}{X_{\text{:}d}}

Given inputs $X\in \mathbb{R}^{N \times D}$, targets $y\in \mathbb{R}^N$, a positive integer $k$, and $\{a_i\}_{i=0}^{k-1}$ with $a_i\in \mathbb{R}$, the residual variance estimate for a subset of features denoted by $s\in \{0,1\}^D$ for the regression model of \cite{Kong2018estimating} is given by
\begin{equation}
    f(s) = \frac{y^\top y}{N} - \sum_{i=0}^{k-1} \frac{a_i}{\binom{N}{i+2}} y^\top \text{triud}(X \text{diag}(s) X^\top)^{i+1} y,
    \label{eq:estimator}
\end{equation}
where $\text{triud}(\cdot)$ denotes the operation that zeros out the lower triangular portion and diagonal entries of a square matrix.
The assumptions for the sublinear learnability bound include i.i.d.\ examples/labels and bounded moments\footnote{Specifically, the 2nd and 4th order moments for the examples and variance for the noise are assumed to be bounded.}.
The parameter $k$ and coefficients $\{a_i\}$ are related to estimator order.
From \cref{eq:estimator}, it is clear that increasing $k$ results in higher-order interaction terms between features being taken into account in the residual variance computation.
First, we show that \Cref{eq:estimator} can be efficiently computed: %
\begin{lemma}
    \label{alg-cost}
    The function \Cref{eq:estimator} requires $\mathcal{O}(ND)$ time and $\mathcal{O}(N)$ space.
\end{lemma}
The following proposition is useful in proving \Cref{alg-cost}.
\begin{proposition}
    \label{basic-feature-ops}
    For $z\in \mathbb{R}^N$ %
    and $y\in \mathbb{R}^N$,
    the operation
    $\text{triud}(z z^\top)y$
    requires $\mathcal{O}(N)$ time and space.
\end{proposition}
\begin{proof}
    Expanding the first operation, we have
    \begin{equation*}
        \text{triud}(z z^\top)y
        =
        \begin{pmatrix}
            z_{1}z_{2}y_2 + z_{1}z_{3}y_3 + \dots + z_{1}z_{N}y_N \\
            z_{2}z_{3}y_3 + \dots + z_{2}z_{N}y_N \\
            \vdots \\
            z_{N-1}z_{N}y_N \\
            0
        \end{pmatrix}.
    \end{equation*}
    Letting $u=z \circ y$,
    \begin{equation*}
        \text{triud}(z z^\top)y
        =
        \begin{pmatrix}
            z_{1} \sum_{i=2}^N u_{i} \\
            z_{2} \sum_{i=3}^N u_{i} \\
            \vdots \\
            z_{N-1} u_{N} \\
            0
        \end{pmatrix},
    \end{equation*}
    where the elements $(\sum_{i=2}^N u_{i}, \sum_{i=3}^N u_{i}, \dots, u_{N}, 0 )$ can be calculated from $u$ with a (reverse) cumulative sum in $\mathcal{O}(N)$.
    Since computing $u$ and performing the inner product between $z$ and $(\sum_{i=2}^N u_{i}, \sum_{i=3}^N u_{i}, \dots, u_{N}, 0 )^\top $ are also $\mathcal{O}(N)$, the total cost is linear in $N$.
\end{proof}

\begin{proof}[Proof of \Cref{alg-cost}]
For $X\in \mathbb{R}^{N \times D}$, let $\xd$ for $1\le d \le D$ represent the $d$-th column of $X$.
Then, note
\begin{align}
    y^\top \text{triud}(X \text{diag}(s) X^\top)^{i+1} y
    & =
    y^\top
    \Big(
        \sum_{d=1}^D s_d \text{triud}(\xd \xd^\top)
    \Big)^{i+1}
    y
    \\
    & = 
    y^\top
    \underbrace{
        \Big(
            \sum_{d=1}^D s_d G_d
        \Big)
        \cdots
    }_{i~\text{terms}}
    \Big(
        \sum_{d=1}^D s_d G_d
    \Big)
    y
    ,
\end{align}
where $G_d\triangleq \text{triud}(\xd \xd^\top)$. %
From \Cref{basic-feature-ops}, it follows that computing $\sum_{d=1}^D s_d G_d y$ requires $\mathcal{O}(ND)$ time and $\mathcal{O}(N)$ space.
This operation can be iterated $i$ additional times without increasing the time or space complexity.
\end{proof}

Therefore, \Cref{eq:estimator} can be utilized as an efficiently computable filter criterion.
However, given a prescribed number of desired features $1\le d\le D$, the use of \Cref{eq:estimator} for \emph{search} over
$s\in\{0,1\}^D$ subject to $\Vert s \Vert_0 = d$
is limited since the general problem is NP-hard \citep{Natarajan1995sparse}.
To achieve an approximate, tractable solution, we relax the discrete search domain to a continuous domain via the transformation $s=\sigma(v)$ where $v\in\mathbb{R}^D$ and the ``squashing'' function $\sigma(\cdot):\mathbb{R}\mapsto[0,1]$ is applied element-wise.
The relaxed, unconstrained optimization is then given by
$    \argmin_{v\in \mathbb{R}^D} f(\sigma(v))$,
which is amenable to solution by gradient-based search in $v$.
Further, mini-batches of data can be utilized for increased efficiency. %
In our experiments, we enforce model parsimony with
the penalty term $\tfrac{\lambda}{D} \Vert s \Vert_1=\tfrac{\lambda}{D}\sum_{d=1}^D \sigma(v_d)$ where
selection of the regularization parameter $\lambda>0$ can be accomplished with grid search evaluated on a validation set. %

We note that (i) the classification bound of \cite{Kong2018estimating} assumes Gaussianity of the examples, and (ii) the objective (plus penalty) is non-convex and likely has many local optima.
However, as we demonstrate in the experiments, the algorithm is able to identify better feature subsets than competing methods in practical applications, where these conditions may not hold. %
\section{Related work}

The research on feature selection is vast, and we cannot hope to provide a list here with any claim to being comprehensive.
Instead, we highlight relevant work with a focus on more recent scalable filter approaches.
The interested reader can refer to \cite{Guyon2003introduction} for a broader overview and
\cite{Saeys2007review} and \cite{Forman2003extensive} for domain-specific reviews in key application areas. %

As mentioned previously, wrapper methods generally produce high quality feature subsets since they optimize directly for the desired supervised learning task.
However, here, we seek a general scalable solution that can support arbitrary follow-on tasks.
Moreover, in the $N\ll D$ regime, approaches that attempt to train a model risk overfitting or underfitting.
In contrast, we are proposing method that produces high-quality estimates of feature subset performance.

Standard filter criteria for classification include the one-way ANOVA test among class means and mutual information (computed per feature dimension).
Like \alg{}, the ANOVA computation is linear in sample size and dimension, but is ``greedy'' in selection in that features are evaluated in isolation.

The filter algorithm of \cite{Yu2003feature,Yu2004efficient} utilizes a normalized mutual information as correlation measure to first identify target-correleated features, and then, from this set, perform rounds of elimination of redundant features, where redundant is also defined w.r.t.\ the same correlation measure.
In this sense, the method does consider pairwise feature interactions, but to \emph{eliminate} features, not to \emph{identify} groups of correlated features.
Assuming an average case performance of eliminating half of all target-correlated features per round, the algorithm complexity is $\mathcal{O}(N D\log D)$; the worst-case complexity is $\mathcal{O}(N D^2)$.

Concrete and related relaxations developed for latent variable models \citep{Jang2016categorical,Maddison2016concrete} offer an avenue for embedded feature selection.
Of these, the closest work is that of \cite{Abid2019concrete} which presents an unsupervised approach that selects a prescribed number $d$ of features by learning $d$ (parallel) Concrete distributions over the input features.
Training an autoencoding neural network architecture yields a subset of features that are useful for reconstructing the full-dimensional inputs, and, to an extent, supporting follow-on supervised learning tasks.
However, a neural network architecture must be specified per problem instance and trained.
Convex relaxations of ``$\ell_0$''-regularized approaches (e.g., \cite{Tibshirani1996regression}) relax the objective function rather than the search space.
Our particular relaxation is more in line with that of \cite{Liu2018darts}.

Of the recent approaches for streaming feature selection \citep{Sun2009local,Yu2014towards}, the MISSION algorithm \citep{Aghazadeh2018mission} most closely considers the same setting we do.
MISSION utilizes a Count-Sketch data structure to store a running, compressed gradient of a generalized linear model which provides scalability to very high dimensional datasets.
However, the gradient information does not take into account higher-order feature interactions.

\section{Experiments}
We validate the performance of our method for classification tasks in both small and large datasets.
Small datasets allow us to carefully examine the performance of the different algorithms we consider in cases where $N\ll D$, whereas the larger datasets are useful to demonstrate scalability to large sample size and large dimension, in both the $N>D$ and $N<D$ regimes.
In the small datasets setting, we compare Feature Gradients to traditional filter methods%
{. In} the large-scale setting %
{, where traditional filter methods would be too computationally expensive,} we compare to MISSION \citep{Aghazadeh2018mission}, a recently-proposed algorithm that can perform efficient feature selection by means of count sketching.
The datasets used in these experiments are summarized in \Cref{tab:datasets}.

The evaluation methodology consisted of first running our method and the competitors on train to yield feature subsets of varying sizes. Then, a logistic regression model is fit per subset with train features (batch-trained for small datasets, SGD-trained for large datasets). Finally, the fitted models are evaluated on test with AUC as the performance metric. In this way, performance curves of AUC as a function of feature subset size are recorded per dataset.

\paragraph{\alg{} implementation.}
Prior to performing the continuous search, we
(i) center the data matrix, $X$, per dimension, %
and
(ii) estimate the largest singular value of the covariance of $X$ and divide the centered data matrix by the square root of this value.
Both of these operations are performed in order to meet the requirements of the \cite{Kong2018estimating} estimator.
Specifically, dividing by the square root of the largest singular value allows a single set of degree-$k$ polynomial coefficients to be used for all input data matrices.
The centering and singular value estimation are performed on the entire data matrix except in the cases of \emph{webspam} and \emph{criteo} where it is estimated from a sub-sample of size 10,000.
Adam \citep{Kingma2014adam} with learning rate $10^{-1}$ (and other parameters set to default) was used in all cases.
For the small data experiments, the stopping criteria were 1000 maximum iterations or the relative change in objective function dropping below $10^{-5}$. %
In the large-scale experiments, training was performed for one epoch.
Mini-batch sizes are described in the relevant sections.
The squashing function was $\sigma(2x)$ where $\sigma(\cdot)$ is the sigmoid function (this choice of squashing function is equivalent to $\tfrac{1}{2}($tanh$(x)+1)$).

\begin{table}
\begin{center}
\begin{tabular}{lccc}
    Dataset & $D$ & $N_\text{train}$ & $N_\text{test}$\\
    \hline
    \emph{mnist35} & 784 & 11,552 & 1902\\
    \emph{gisette} & 5000 & 6000 & 1000\\
    \emph{rcv1} & 47,236 & 20,242 & 677,399\\
    \emph{webspam} & 16,609,143 & 280,000 & 70,000\\
    \emph{criteo} & 1,000,000 & 45,840,617 & 6,042,135\\
\end{tabular}
\caption{Binary classification datasets.}
\label{tab:datasets}
\end{center}
\end{table}

\paragraph{Small datasets.}

The smaller dimensional datasets are:
\begin{itemize}
    \item \emph{mnist35}. %
        Classification of 3s vs. 5s in the MNIST dataset with random noise added to produce a more challenging scenario.
        Samples are shown in Figure \ref{fig:exp-mnist-example}.
    \item \emph{gisette}. Used in the 2003 NIPS feature selection challenge.
\end{itemize}

Beyond benchmarking the performance of our method, the goal of this experiment is to provide an intuitive visual representation of the features selected by \alg{}.
We first discuss performance on the \emph{mnist35} dataset.

\begin{figure}
    \centering
    \includegraphics[width=0.18\linewidth]{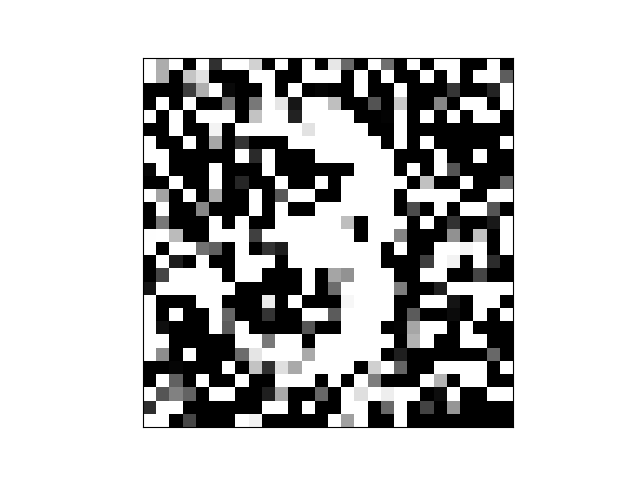}
    \includegraphics[width=0.18\linewidth]{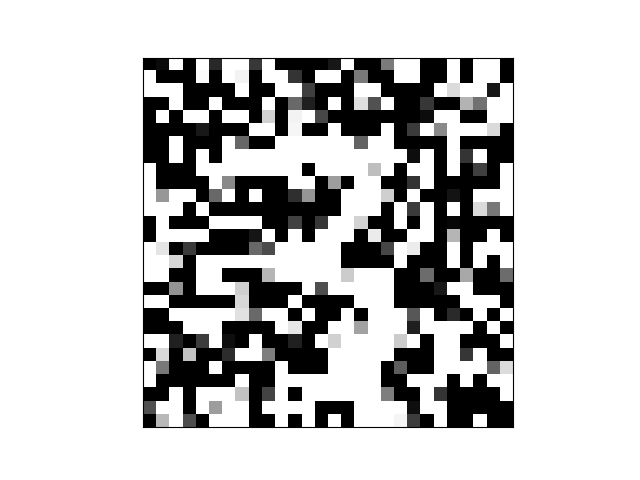}
    \includegraphics[width=0.18\linewidth]{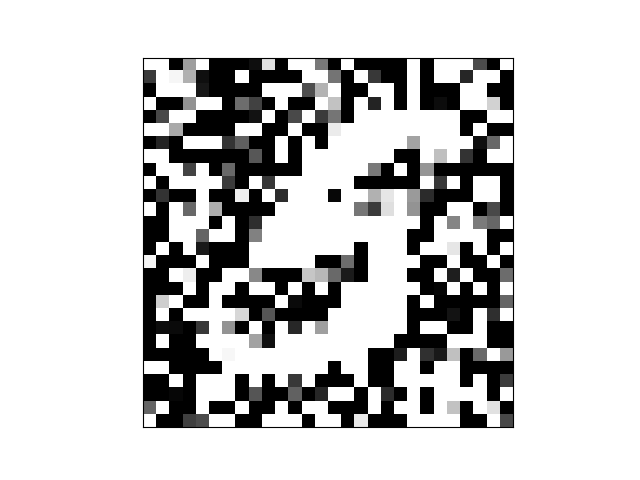}
    \includegraphics[width=0.18\linewidth]{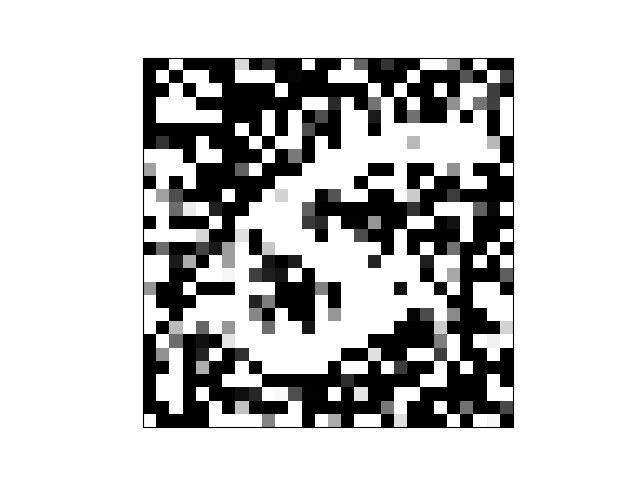}
    \includegraphics[width=0.18\linewidth]{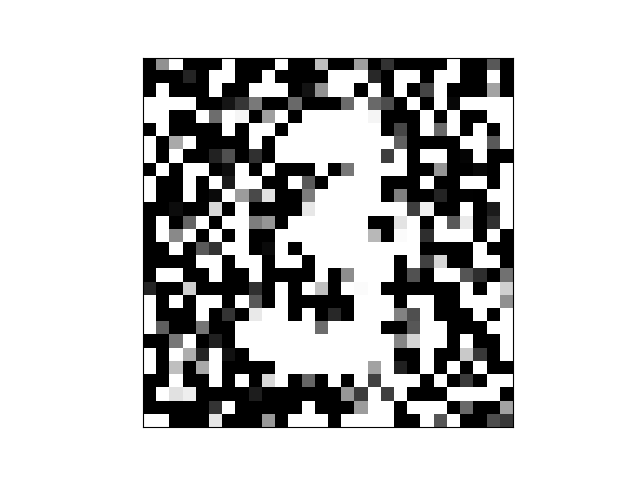}
    \\
    \includegraphics[width=0.18\linewidth]{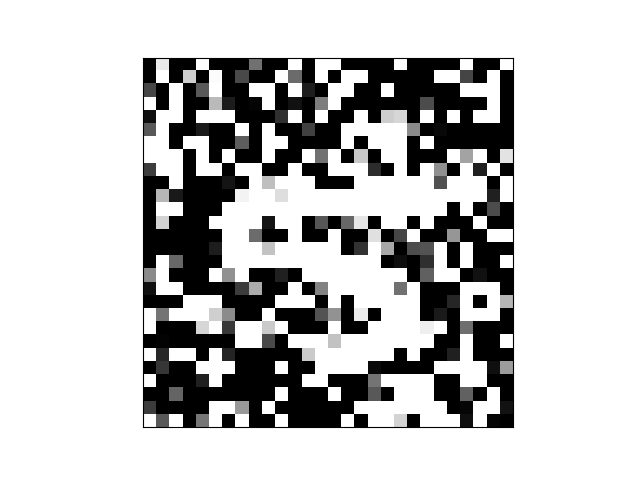}
    \includegraphics[width=0.18\linewidth]{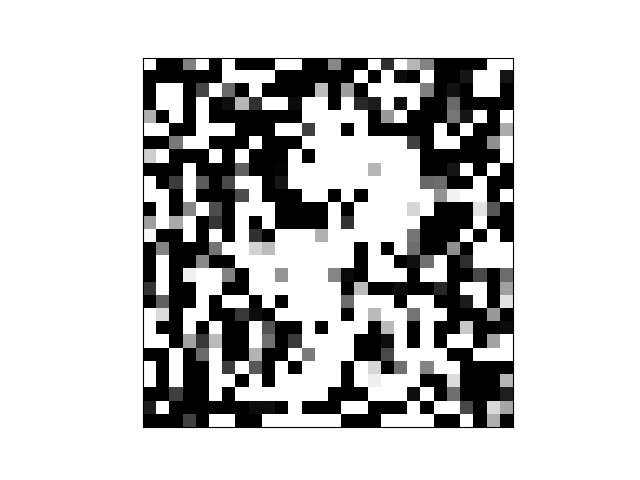}
    \includegraphics[width=0.18\linewidth]{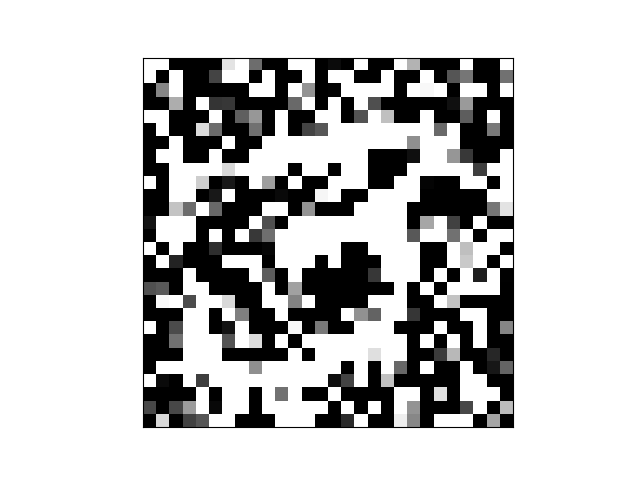}
    \includegraphics[width=0.18\linewidth]{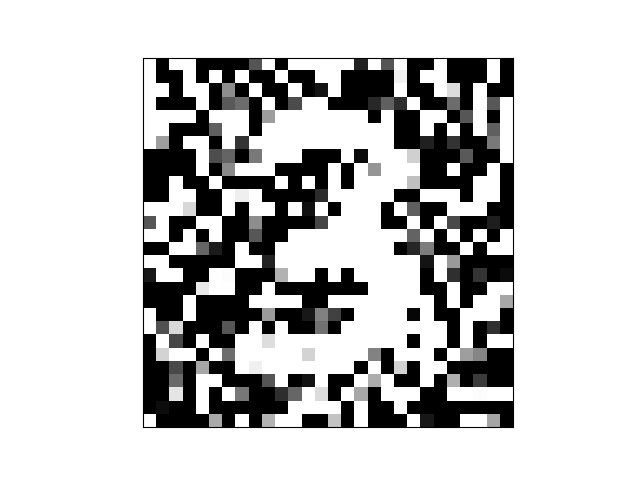}
    \includegraphics[width=0.18\linewidth]{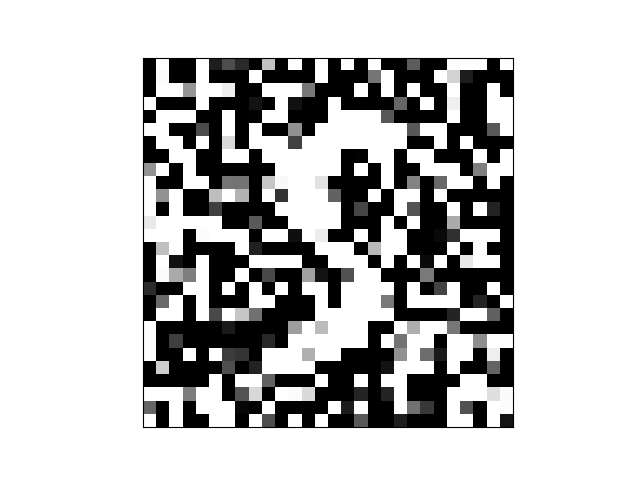}
    \caption{\emph{mnist35} input examples.}
    \label{fig:exp-mnist-example}
\end{figure}

Figure \ref{fig:exp-mnist-gisette} compares the performance of our method with two filter-based methods implemented in scikit-learn.
One computes ANOVA F-values between features and targets (SKF) and the other computes mutual information (SKMI). %
Our method significantly outperforms both baselines across a range of feature subset sizes.
As a whole, the improvement given by \alg{} is statistically significant at a significance level of $0.01$ with a p-value $<10^{-4}$, based on a paired t-test between AUCs obtained by each baseline and AUCs obtained by our method. %

In \Cref{fig:mnist-qual-feat-sel-diff}, we highlight which features within the images are selected by our algorithm. In particular, we plot values of the parameter for the unconstrained optimization, $v$, for different values of $\lambda$, resulting in different amounts of sparsity. Across all the sparsity levels we considered, the features selected by our method cluster in the top half of the image. Indeed, it's easy to see that the most useful features to distinguish the digit ``3'' from the digit ``5'' are those located where the ``upper loop'' in the digit closes or opens. In constrast, the features selected by the filter-based methods tend to cluster together owing to the greedy selection of single features at a time.

Finally, Figure\ref{fig:order} right shows that increasing the order of the estimator results in increasingly better results. In particular, we found that the improvement between order 6 and order 1 is statistically significant at a significance %
{level of $0.05$ with a p-value of $0.043$ for \emph{gisette} and $< 10^{-4}$ for \emph{mnist35}}.

\begin{figure}
    \centering
    \begin{subfigure}{0.475\linewidth}
        \includegraphics[width=\linewidth]{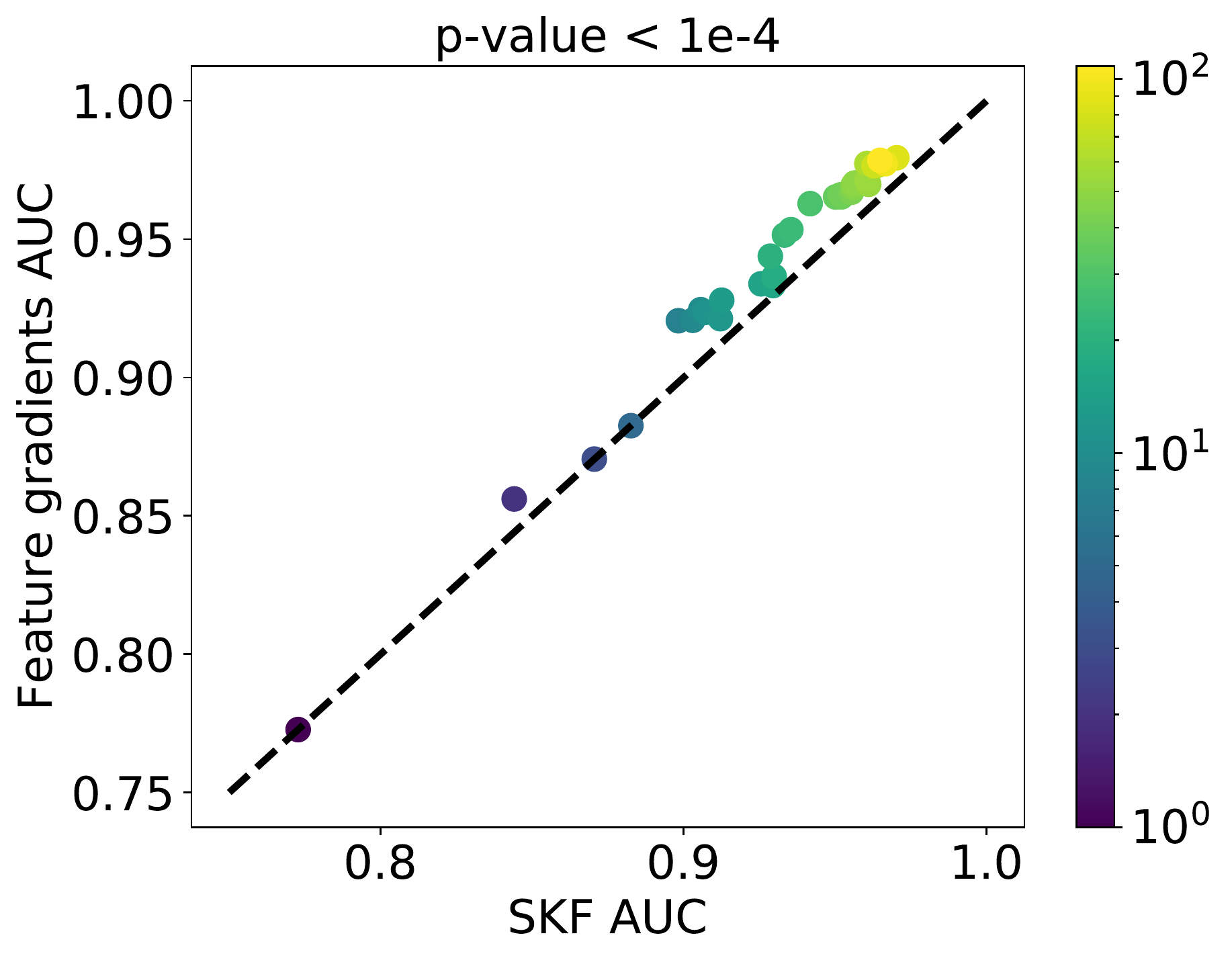}
        \\~\\
    \end{subfigure}
    \begin{subfigure}{0.475\linewidth}
        \includegraphics[width=\linewidth]{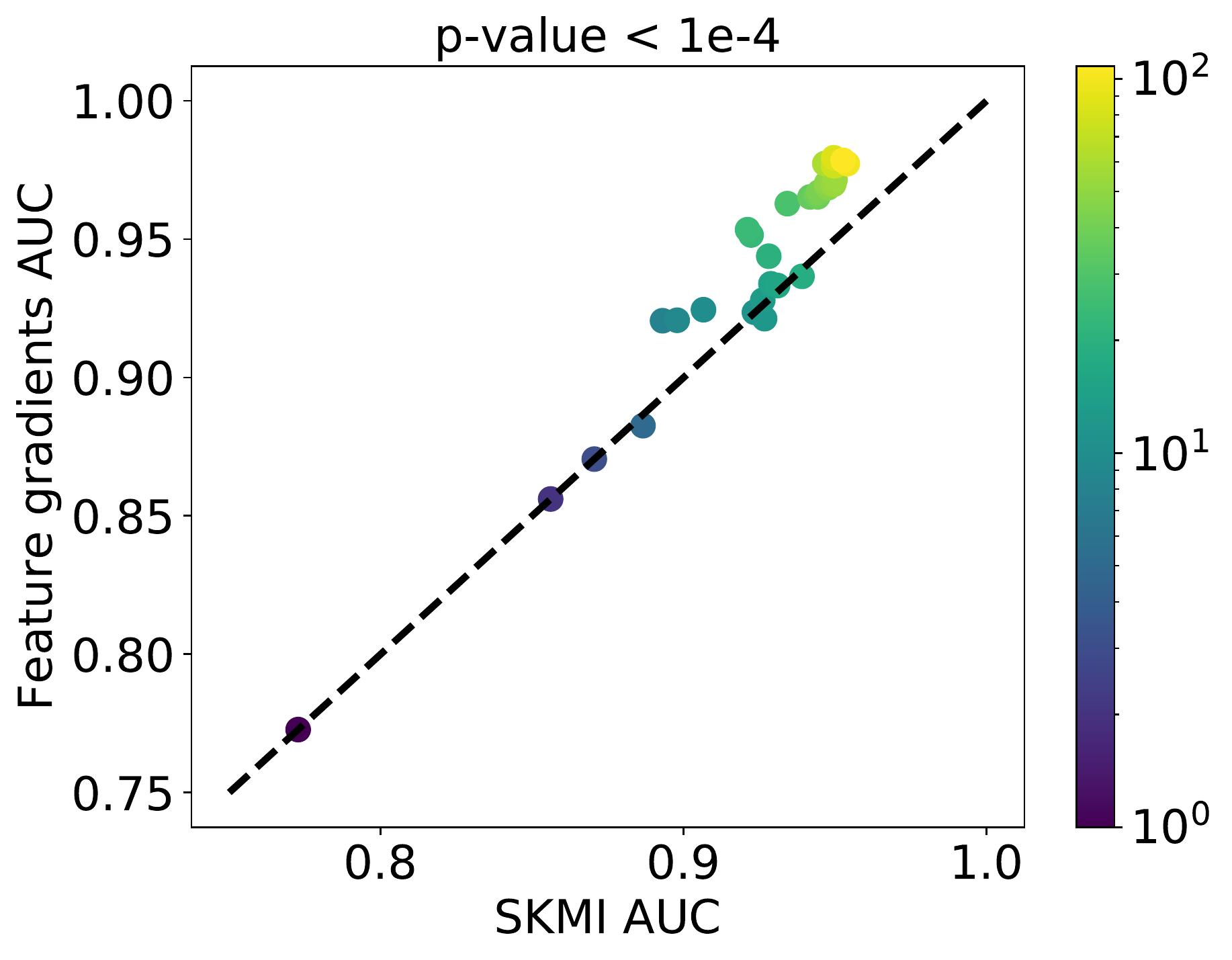}
        \\~\\
    \end{subfigure}
    \\
    \begin{subfigure}{0.475\linewidth}
        \includegraphics[width=\linewidth]{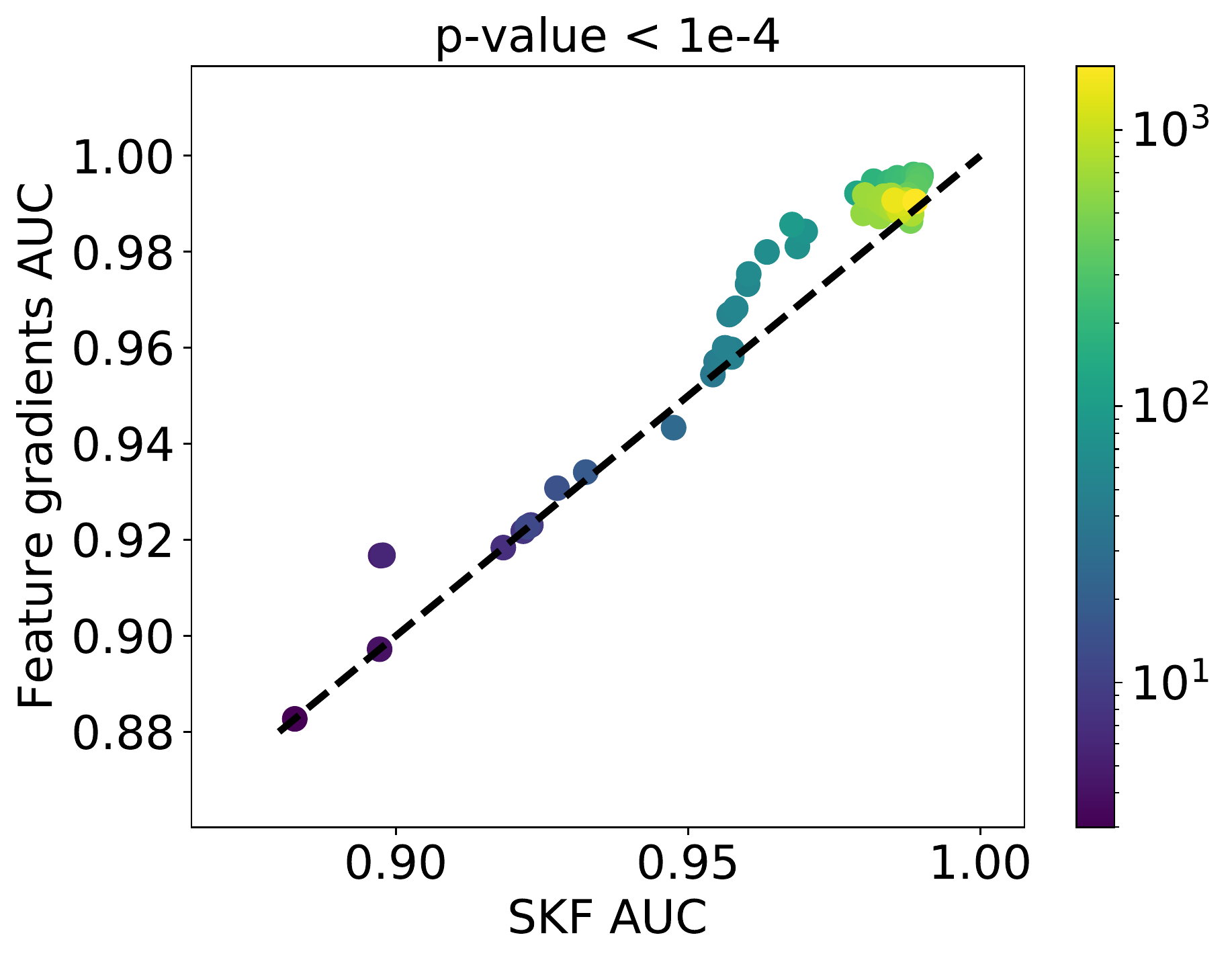}
        \\~\\
    \end{subfigure}
    \begin{subfigure}{0.475\linewidth}
        \includegraphics[width=\linewidth]{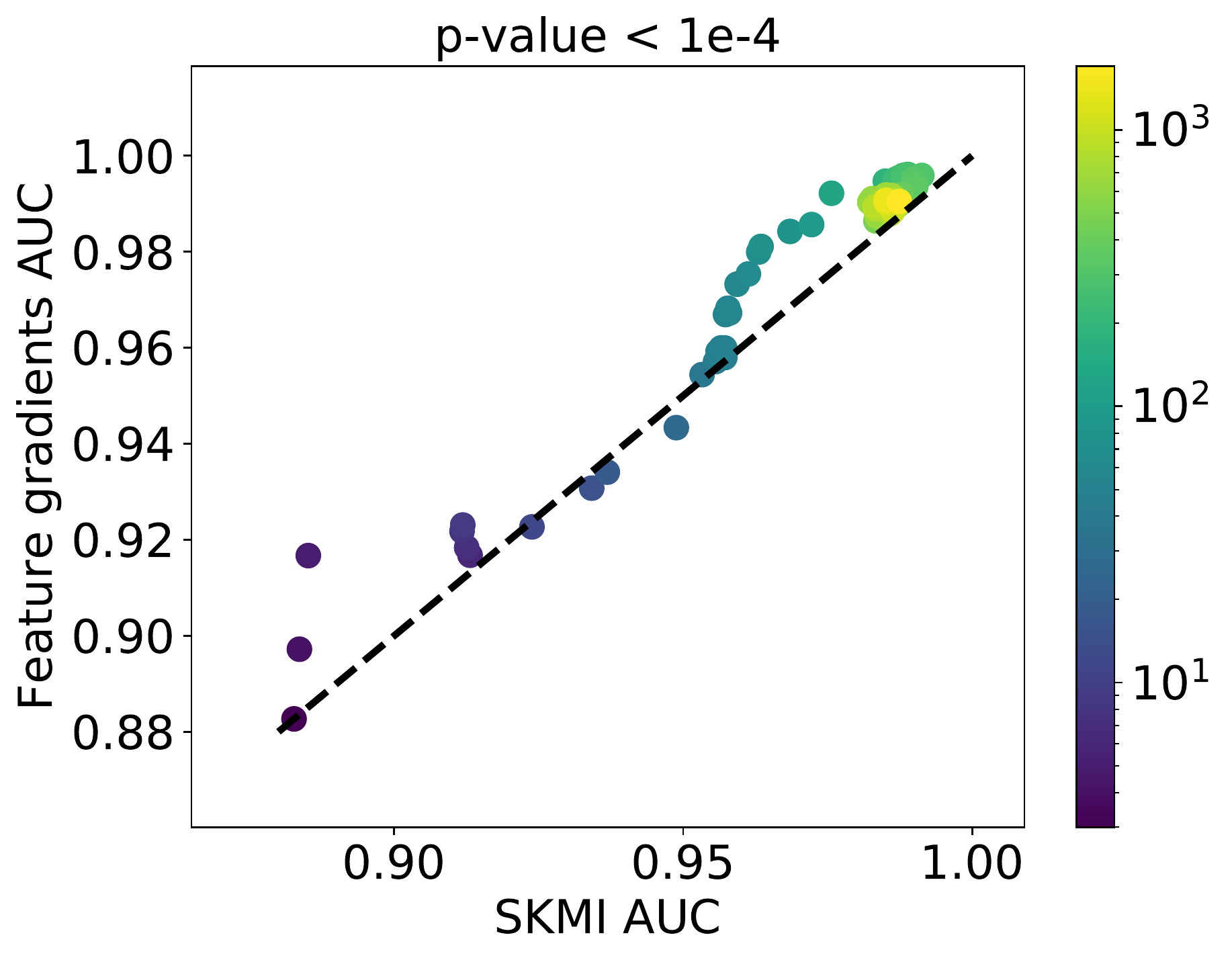}
        \\~\\
    \end{subfigure}
    \caption{\alg{} order 6 performance vs. standard filter methods SKF (left) and SKMI (right) on \emph{mnist35} (top) and \emph{gisette} (bottom). Points above the dashed line indicate where FG outperformed the competitor and vice versa below the dashed line. The color of the points denotes the feature subset size.}
    \label{fig:exp-mnist-gisette}
\end{figure}

\begin{figure}
    \centering
    \includegraphics[width=\linewidth]{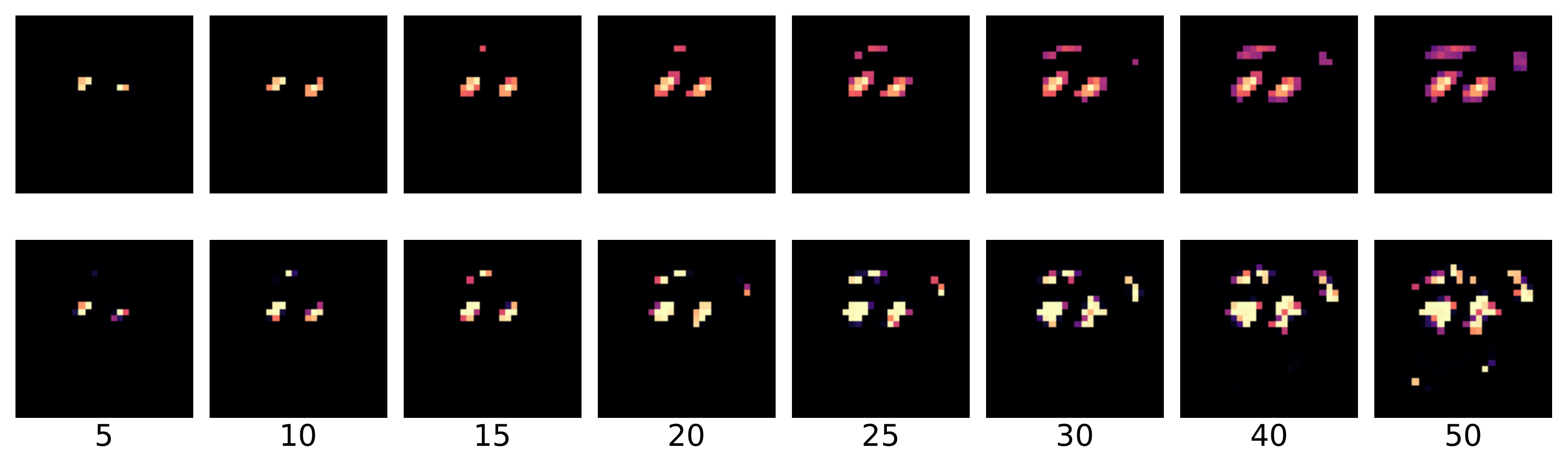}
    \caption{Feature selection on \emph{mnist35}. The different columns represent different feature subset sizes (denoted below the bottom row). The features selected by the filter method utilizing the one-way ANOVA criterion (top) and by \alg{} order 6 (bottom) are colored according to their (normalized) scores.}
    \label{fig:mnist-qual-feat-sel-diff}
\end{figure}

\begin{figure}
    \centering
    \includegraphics[width=0.450\linewidth]{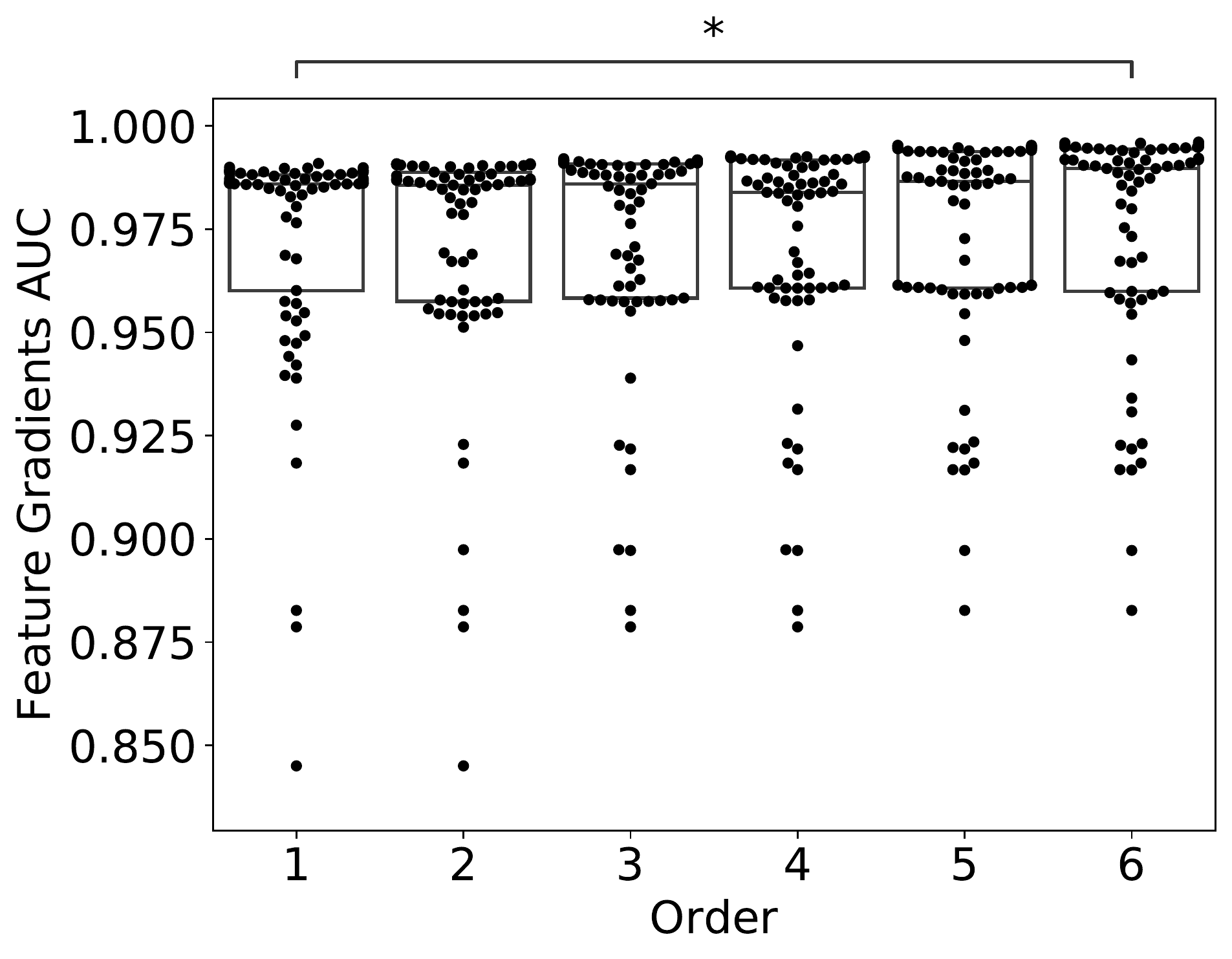}
    \includegraphics[width=0.450\linewidth]{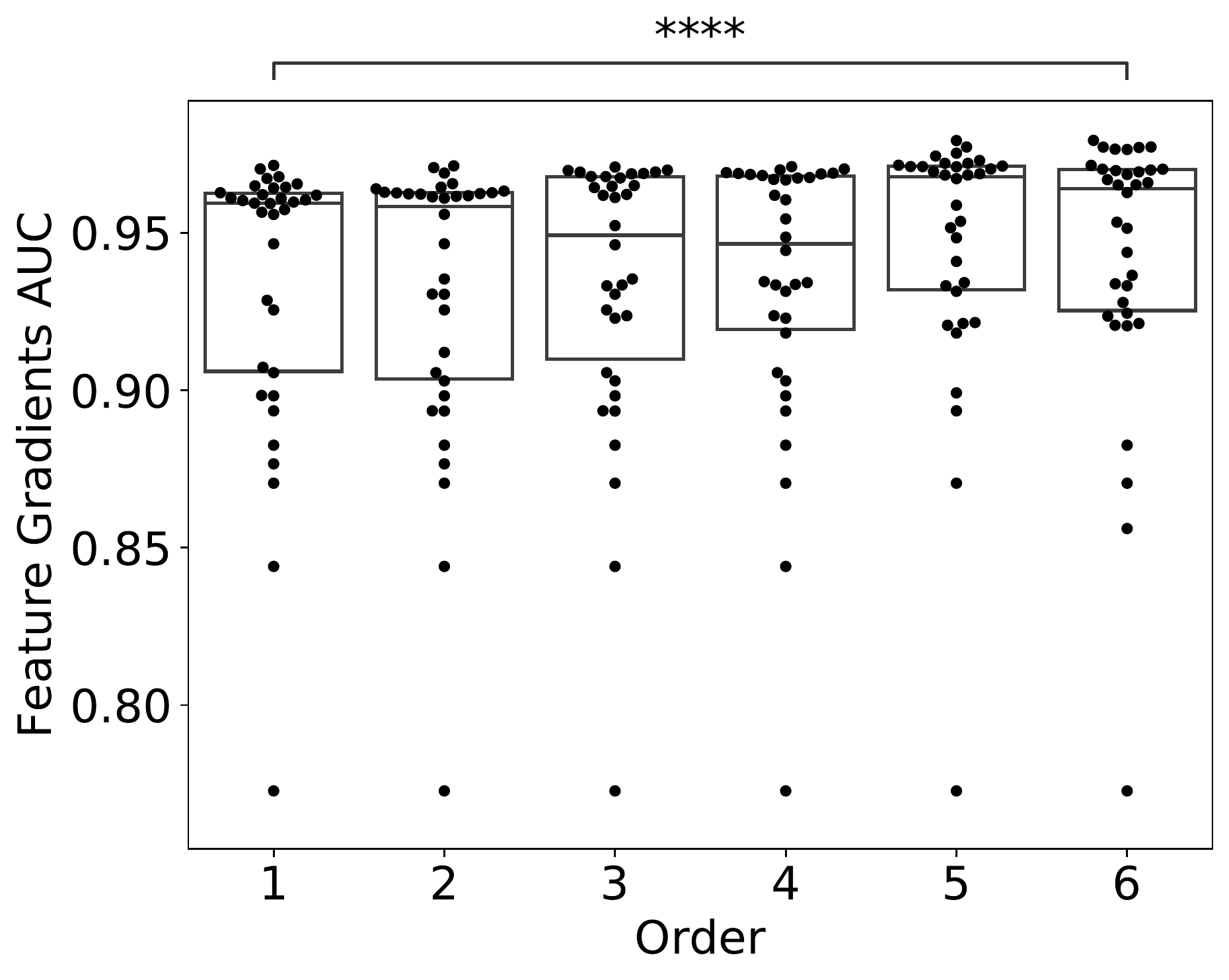}\\
    \caption{%
    {Scores (y-axis) as a function of \alg{} order (x-axis) on \emph{gisette} (left) and \emph{mnist35} (right). The statistical significance of the improvement between order 1 and order 6 was measured using a paired t-test and the p-values are represented with a star at the top of the plot (p-value is $0.043$ for \emph{gisette} and $< 10^{-4}$ for \emph{mnist35})}.}
    \label{fig:order}
\end{figure}

We repeat the same experiments for \emph{gisette}, an even more challenging dataset in which the number of features is close to the number of samples available.
The results in Figure \ref{fig:exp-mnist-gisette} bottom left and right show that our method %
{shows significantly better performance than} the baselines we considered.
As was the case for the previous dataset, increasing the order of \alg{} significantly increases performance (Figure\ref{fig:order} left).

\paragraph{Large-scale experiments.}
Here, we compare \alg{} order 4 against MISSION on one moderate-sized dataset \emph{rcv1} and two large datasets \emph{webspam} and \emph{criteo}.
In the case of \emph{rcv1} a mini-batch size of 1000 was used.
For \emph{webspam}, the mini-batch size was 8, and for \emph{criteo} it was 100.
In these two latter cases, the mini-batch size was selected to be the largest that would fit into GPU memory on an NVIDIA Tesla P100.
Also, in the latter two datasets, gradients were accumulated up to a size of 1000 examples.
After training FG for one epoch to learn the feature subsets, the MISSION code was used to train and evaluate an SGD classifier using the selected features.
To support a fair comparison, the authors' code was modified to support this evaluation\footnote{Specifically, we modified the code to allow a set of features to be input prior to test.}.
MISSION was then also run for the selected number of features.

Figure \ref{fig:exp-rcv-webspam} shows the performance of \alg{} order 4 vs. MISSION on \emph{rcv1} (top left) and \emph{webspam} (top right).
For both datasets, \alg{} is able to locate feature subsets resulting in test set performance significantly better than MISSION (significance level 0.01, with p-values of $0.0064$ and $0.0201$ for \emph{rcv1} and \emph{webspam}, respectively). Notably, for the webspam dataset, \alg{} order 4 is able to maintain 0.92 AUC on test with just 22 features selected out of over 16 million.
\begin{figure}
    \centering
    \includegraphics[width=0.475\linewidth]{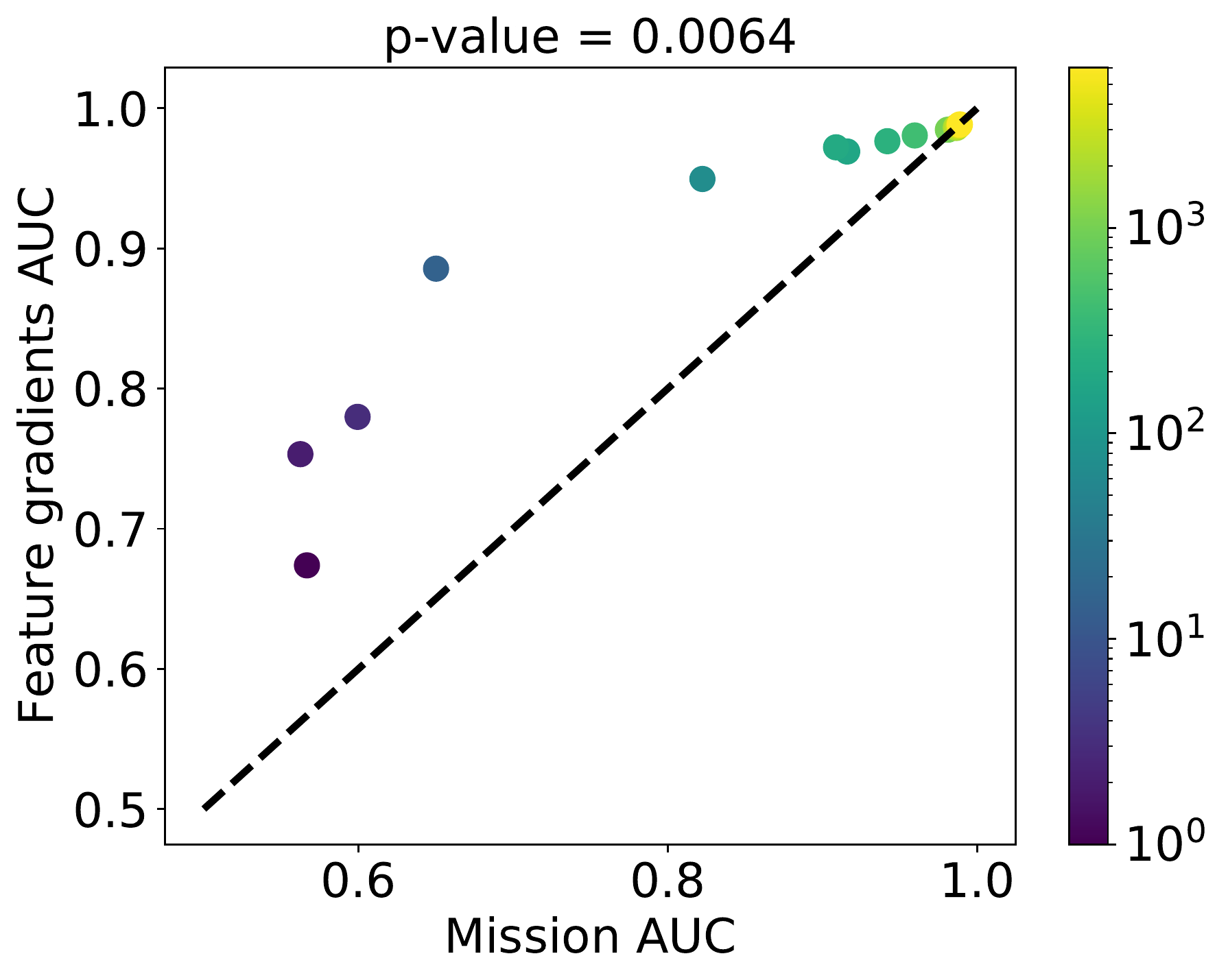}
    \includegraphics[width=0.475\linewidth]{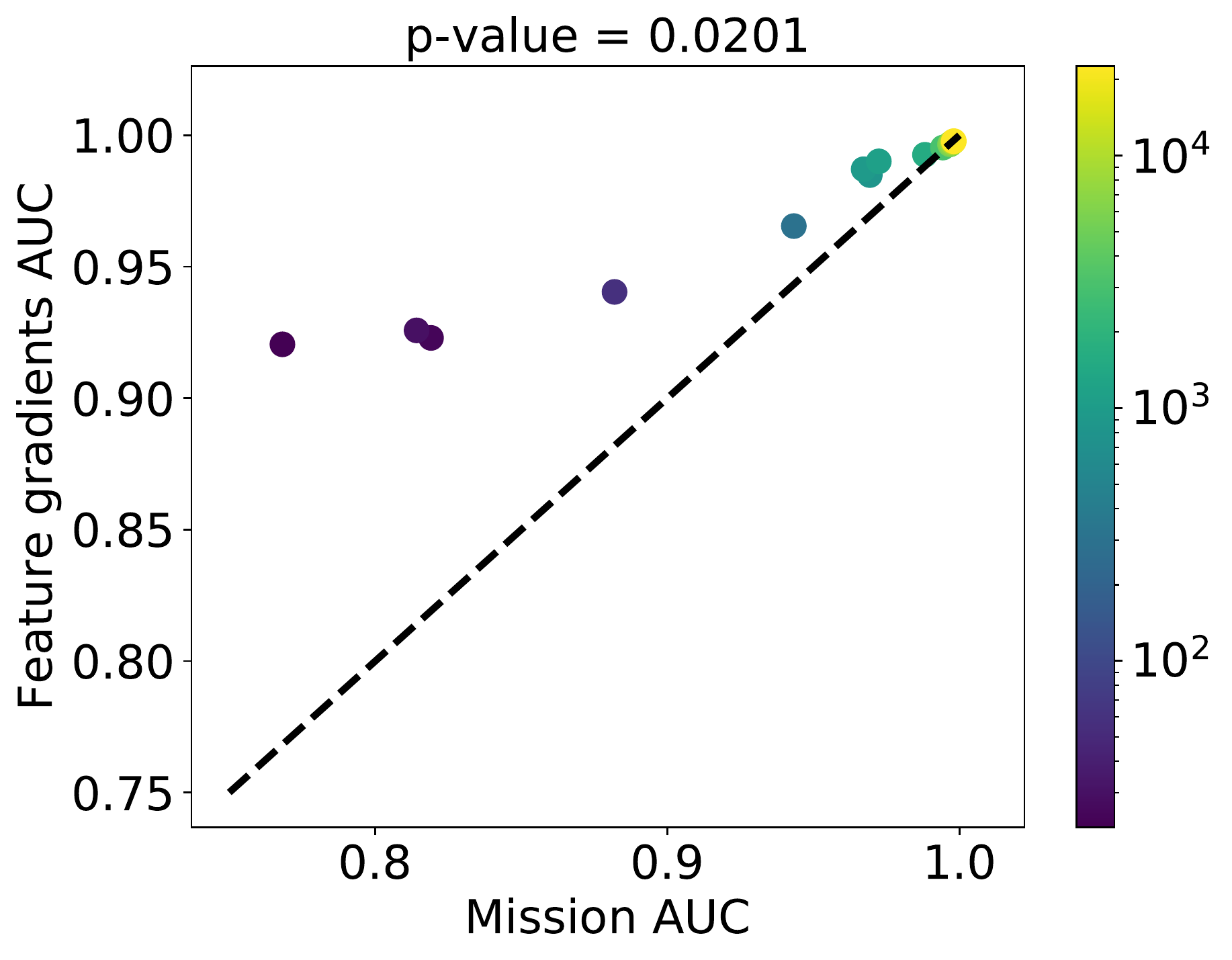}\\
    \includegraphics[width=0.475\linewidth]{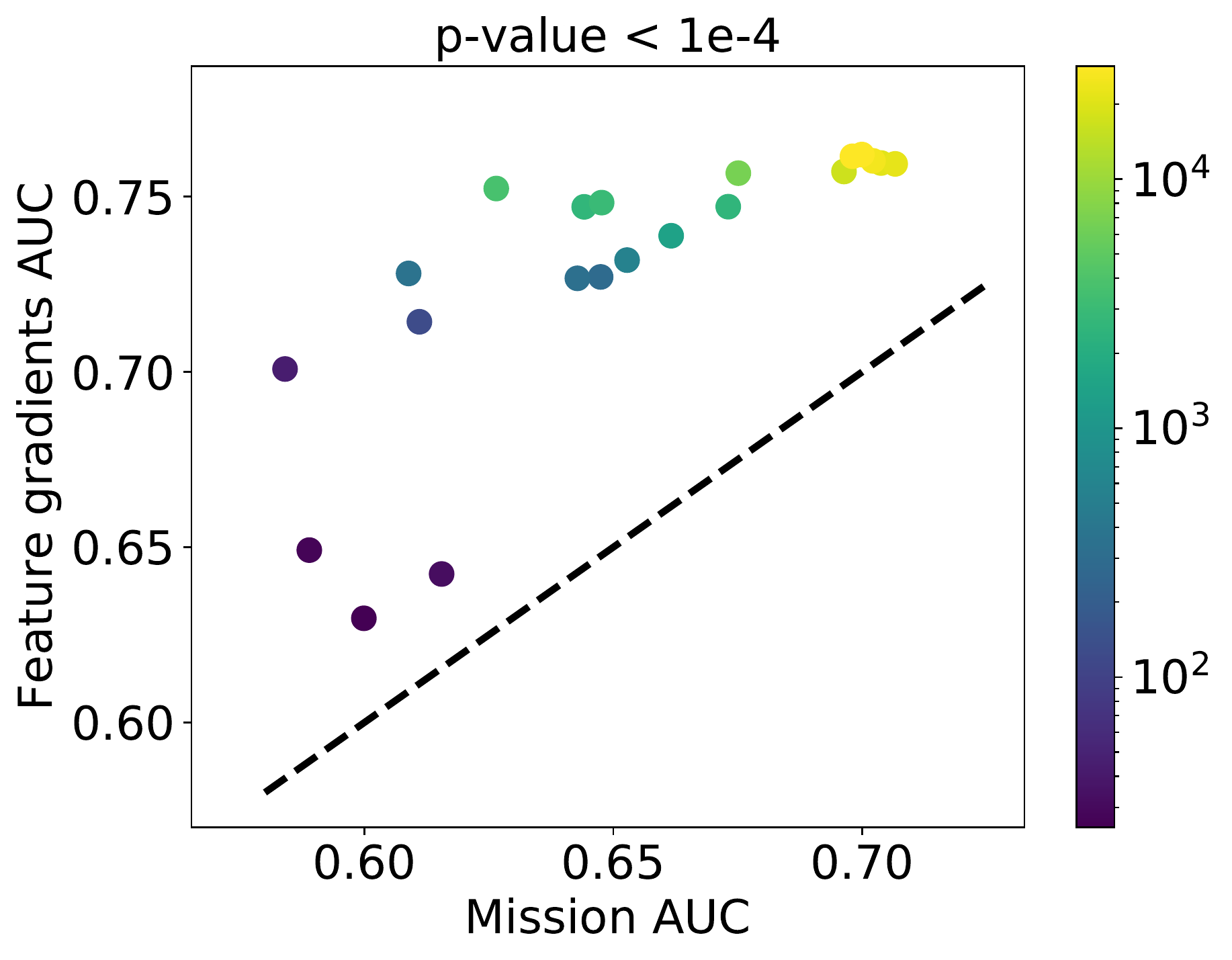}
    \caption{\alg{} order 4 performance vs. MISSION on \emph{rcv1} (top left), \emph{webspam} (top right), and \emph{criteo} (bottom).}
    \label{fig:exp-rcv-webspam}
\end{figure}

Finally, as shown in Figure \ref{fig:exp-rcv-webspam} bottom, \alg{} significantly (p-value $<10^{-4}$) outperforms MISSION on \emph{criteo} across \emph{all} the sparsity levels (\ie number of features) considered.

\section{Discussion}

As the experiments show, our proposed method, Feature Gradients, results in feature subsets that are more correlated with %
{the} targets than standard filter-based methods on small datasets as well as state-of-the-art streaming feature selection on large datasets.
The main limitation of FG lies in the data centering step performed per mini-batch, which %
{results} in large memory utilization. %
{In memory-constrained environments (\eg when using GPUs to accelerate computations), this issue can be mitigated by accumulating gradients. However, avoiding the data centering step entirely would result in better memory usage and faster computation, %
for example, when dealing with sparse data. We will focus on this aspect in future work.}

Feature Gradients can be trivially extended to the multi-classification setting via the standard softmax relaxation applied across multiple outputs. %
Additionally, our method is also capable of ``online'' feature selection, where instances are presented one at a time.
Preliminary results (not shown) indicate that FG is capable of tracking which feature subsets are most correlated with the target as a function of time on \emph{criteo}.
A promising direction would be to combine FG with dynamic feature engineering wherein features are %
{iteratively} constructed, %
{rather than first expanded and then filtered, as done in this paper}.

\bibliography{fg}
\bibliographystyle{abbrvnat}

\newpage
\clearpage

\end{document}